\documentclass[preprint,12pt]{elsarticle}
\usepackage{mdwmath}
\usepackage{mdwtab}
\usepackage{amsfonts,amsmath,amsthm,multirow,enumerate}
\usepackage[francais,english]{babel}
\newtheorem{remark}{Remark}

\usepackage{graphicx}
\usepackage{subfigure}
\usepackage{algorithm}
\usepackage{algorithmic}
\usepackage{amssymb}
\usepackage{amscd}

\newtheorem{proposition}{Proposition}
\newtheorem{assumption}{Assumption}
\newtheorem{corollary}{Corollary}

\newtheorem{theorem}{Theorem}
\newtheorem{example}{Example}
\def\blackslug{\hbox{\hskip 1pt \vrule width 4pt height 8pt depth 1.5pt
\hskip 1pt}}
\def\qed{\quad\blackslug\lower 8.5pt\null\par\medskip}

\journal{Journal of Multivariate Analysis}

\begin{document}

\begin{frontmatter}

\title{ Survey schemes for stochastic gradient descent \\with applications to $M$-estimation}

\author[TPT]{St\'{e}phan Cl\'{e}men\c{c}on}
       
       \author[PX]{Patrice Bertail}
       
       \author[MPT]{Emilie Chautru}
       
       \author[TPT]{Guillaume Papa}
       
      \address[TPT]{Telecom ParisTech LTCI UMR Telecom ParisTech/CNRS No. 5141\\
             Telecom ParisTech
             46 rue Barrault, Paris, 75634, France}
      \address[PX]{Universit\'e Paris Ouest\\
                    MODAL'X
                    200 avenue de la R\'epublique, Nanterre, 92000, France}
                    \address[MPT]{Mines ParisTech\\
                                        Centre de Geosciences
                                         35 rue Saint Honor\'e, Fontainebleau, 77305, France}

\begin{abstract}
In certain situations that shall be undoubtedly more and more common in the Big Data era, the datasets available are so massive that computing statistics over the full sample is hardly feasible, if not unfeasible. A natural approach in this context consists in using survey schemes and substituting the "full data" statistics with their counterparts based on the resulting random samples, of manageable size. It is the main purpose of this paper to investigate the impact of survey sampling with unequal inclusion probabilities on stochastic gradient descent-based $M$-estimation methods in large-scale statistical and machine-learning problems.  Precisely, we prove that, in presence of some \textit{a priori} information, one may significantly increase asymptotic accuracy when choosing appropriate first order inclusion probabilities, without affecting complexity.  These striking results are described here by limit theorems and are also illustrated by numerical experiments. 
\end{abstract}

\end{frontmatter}

\section{Introduction}

In many situations, data are not the sole information that can be exploited by statisticians. Sometimes, they can also make use of weights resulting from some survey sampling design. Such quantities correspond either to true inclusion probabilities or else to calibrated or post-stratification weights, minimizing some discrepancy under certain margin constraints for the inclusion probabilities.  Asymptotic analysis of Horvitz-Thompson estimators based on survey data (see \cite{HT51}) has received a good deal of attention, in particular in the context of mean estimation and regression (see \cite{Hajek64},\cite{Ros72}, \cite{Rob82}, \cite{DS92}, \cite{Ber98} for instance). The last few years have witnessed significant progress towards a comprehensive functional limit theory for distribution function estimation, refer to \cite{Wellner88}, \cite{Breslow06}, \cite{Breslow08}, \cite{Breslow09}, \cite{Wellner11} or \cite{BCC13}. In parallel, the field of machine-learning has been the subject of a spectacular development. Its practice has become increasingly popular in a wide range of fields thanks to various breakout algorithms (\textit{e.g.} neural networks, SVM, boosting methods) and is supported by a sound probabilistic theory based on recent results in the study of empirical processes, see \cite{DGL96}, \cite{Kolt06}, \cite{BBL05}. However, our increasing capacity to collect data, due to the ubiquity of sensors, has improved much faster than our ability to process and analyze Big Datasets, see \cite{BB08}. The availability of massive information in the Big Data era, which machine-learning procedures could theoretically now rely on, has motivated the recent development of \textit{parallelized/distributed} variants of certain statistical learning algorithms, see \cite{BBL11}, \cite{mateos:bazerque:giannakis:2010}, \cite{navia:etal:2006} or \cite{BCJM13} among others. It also strongly suggests to use survey techniques, as a remedy to the apparent intractability of learning from datasets of explosive size, in order to break the current computational barriers, see \cite{CRT13}. The present article explores the latter approach, following in the footsteps of \cite{CRT13}, where the advantages of specific sampling plans compared to naive sub-sampling strategies are proved when the risk functional is estimated by generalized $U$-statistics.
 
 Our goal is here to show how to incorporate sampling schemes into iterative statistical learning techniques based on stochastic gradient descent (SGD in abbreviated form, see \cite{B98}) such as {\sc SVM}, {\sc Neural Networks} or {\sc soft $K$-means} for instance and establish (asymptotic) results, in order to guarantee their theoretical validity. The variant of the SGD method we propose involves a specific estimator of the gradient, which shall be referred to as the \textit{Horvitz-Thompson gradient estimator} ({\sc HTGD} estimator in abbreviated form) throughout the paper and accounts for the sampling design by means of which the data sample has been selected at each iteration.  For the estimator thus produced, consistency and asymptotic normality results describing its statistical performance are established under adequate assumptions on the first and second order inclusion probabilities. They reveal that accuracy may significantly increase (\textit{i.e.} the asymptotic variance may be drastically reduced) when the inclusion probabilities of the survey design are picked adequately, depending on some supposedly available extra information, compared to a naive implementation with equal inclusion probabilities. This is thoroughly discussed in the particular case of the Poisson survey scheme. Although it is one of the simplest sampling designs, many more general survey schemes may be expressed as Poisson schemes conditioned upon specific events.
 We point out that statistical learning based on non i.i.d. data has been investigated in \cite{Steinwart09} (see also \cite{Duchi13} for analogous results in the on-line framework). However, the framework considered by these authors relies on \textit{mixing} assumptions, guaranteeing the weak dependency of the data sequences analyzed, and is thus quite different from that developed in the present article. We point out that a very preliminary version of this work has been presented at the 2014 IEEE International Conference on Big Data.
 
 The rest of the paper is structured as follows. Basics in $M$-estimation and SGD techniques together with key notions in survey sampling theory are briefly recalled in section \ref{sec:background}. Section \ref{sec:HTalgo} first describes the Horvitz-Thompson variant of the SGD in the context of a general $M$-estimation problem. In section \ref{sec:main}, limit results are established in a general framework, revealing the possible significant gain in terms of asymptotic variance resulting from sampling with unequal probabilities in presence of extra information.  They are next discussed in more depth in the specific case of Poisson surveys. 
 Illustrative numerical experiments, consisting in fitting a logistic regression model (respectively, a semi-parametric shift model ) with extra information, are displayed in section \ref{sec:num}. 
 Technical proofs are postponed to the Appendix section, together with a rate bound analysis of the {\sc HTGD} algorithm.

\section{Theoretical background}\label{sec:background}
 As a first go, we start off with describing the mathematical setup and recalling key concepts in survey theory involved in the subsequent analysis. Here and throughout, the indicator function of any event $\mathcal{B}$ is denoted by $\mathbb{I}\{\mathcal{B}\}$, the Dirac mass at any point $a$ by $\delta_a$ and the power set of any set $E$ by $\mathcal{P}(E)$. The euclidean norm of any vector $x\in \mathbb{R}^d$, $d\geq 1$, is denoted by $\vert\vert x\vert\vert=(\sum_{i=1}^d x_i^2)^{1/2}$. The transpose of a matrix $A$ is denoted by $A^T$, the square root of any symmetric semi-definite positive matrix $B$ by $B^{1/2}$.

\subsection{Iterative $M$-estimation and SGD methods}
Let $\Theta\subset \mathbb{R}^q$ with $q\geq 1$ be some parameter space and $\psi:\mathbb{R}^d\times \Theta\rightarrow \mathbb{R}$ be some smooth loss function. Let $Z$ be a random variable taking its values in $\mathbb{R}^d$ such that $\psi(Z,\theta)$ is square integrable for any $\theta\in \Theta$. Set $L(\theta)=\mathbb{E}[\psi(Z,\theta)]$ for all $\theta\in \Theta$. Consider the \textit{risk minimization} problem
$$
\min_{\theta\in \Theta}L(\theta).
$$
Based on independent copies $Z_1,\; \ldots,\; Z_N$ 	of the r.v. $Z$, the empirical version of the risk function is
$\theta\in \Theta\mapsto \widehat{L}_N(\theta)$,
where $$
\widehat{L}_N(\theta)=\frac{1}{N}\sum_{i=1}^N\psi(Z_i,\theta)
$$
for all $\theta \in \Theta$. As $N\rightarrow +\infty$, asymptotic properties of $M$-estimators, \textit{i.e.} minimizers of $\widehat{L}_N(\theta)$, have been extensively investigated, see Chapter 5 in \cite{vdV98} for instance. Here and throughout, we respectively denote by $\nabla_{\theta}$ and $\nabla^2_{\theta}$ the gradient and Hessian operators w.r.t. $\theta$. By convention, $\nabla^0_{\theta}$ denotes the identity operator and gradient values are represented as column vectors.
\medskip

\noindent {\bf Gradient descent.} Concerning computational issues (see \cite{Berts03}), many practical machine-learning algorithms implement variants of the standard gradient descent method, following the iterations:
\begin{equation}\label{eq:iter1}
\theta(t+1)=\theta(t)-\gamma(t)\nabla_{\theta}\widehat{L}_N(\theta(t)),
\end{equation}
with an initial value $\theta(0)$ arbitrarily chosen and a learning rate (step size or gain) $\gamma(t)\geq 0$ such that $\sum_{t=1}^{+\infty}\gamma(t)=+\infty$ and $\sum_{t=1}^{+\infty}\gamma^2(t)<+\infty$. Here we place ourselves in a large-scale setting, where the sample size $N$ of the training dataset is so large that computing the gradient of $\widehat{L}_N$
\begin{equation}\label{eq:emp_gradient}
\widehat{l}_N(\theta)=\frac{1}{N}\sum_{i=1}^N\nabla_{\theta}\psi(Z_i,\theta)
\end{equation}
at each iteration \eqref{eq:iter1} is too demanding regarding available memory capacity. Beyond parallel and distributed implementation strategies (see \cite{BBL11}), a natural approach consists in replacing \eqref{eq:emp_gradient} by a counterpart computed from a subsample 
$S\subset\{1,\; \ldots,\; N  \}$ of reduced size $n<<N$, fixed in advance so as to fulfill the computational constraints, and drawn at random (uniformly) among all possible subsets of same size:
\begin{equation}\label{eq:sub_gradient_naive}
\bar{l}_n(\theta)=\frac{1}{n}\sum_{i\in S}\nabla_{\theta}\psi(Z_i,\theta).
\end{equation}
The convergence properties of such a stochastic gradient descent, usually referred to as \textit{mini-batch} SGD have received a good deal of attention, in particular in the case $n=1$, suited to the \textit{on-line} situation where training data are progressively available. Results, mainly based on stochastic approximation combined with convex minimization theory, under appropriate assumptions on the decay of the step size $\gamma(t)$ are well-documented in the statistical learning literature. References are much too numerous to be listed exhaustively, see \cite{KYBook} for instance.

\begin{example}{\sc (Binary classification)}
We place ourselves in the usual binary classification framework, where $Y$ is a binary random output, taking its values in $\{-1,\; +1\}$ say, and $X$ is an input random vector valued in a high-dimensional space $\mathcal{X}$, modeling some (hopefully) useful observation for predicting $Y$. Based on training data $\{(X_1,Y_1),\; \ldots,\; (X_N,Y_N)  \}$, the goal is to build a prediction rule $\mathrm{sign}(h(X))$, where $h:\mathcal{X}\rightarrow \mathbb{R}$ is some measurable function, which minimizes the risk
$$L_{\varphi}(h)=\mathbb{E}\left[ \varphi(-Yh(X)) \right],$$
where expectation is taken over the unknown distribution of the pair of r.v.'s $(X,Y)$ and $\varphi:\mathbb{R}\rightarrow [0,\;+\infty)$ denotes a cost function (\textit{i.e.} a measurable function such that $\varphi(u)\geq \mathbb{I}\{u\geq 0\}$ for any $u\in \mathbb{R}$). For simplicity, consider the case where decision function candidates $h(x)$ are assumed to belong to the parametric set of square integrable (with respect to $X$'s distribution) functions indexed by $\Theta\subset \mathbb{R}^q$, $q\geq 1$, $\{h(.,\; \theta),\; \theta\in \Theta  \}$ and the convex cost function is $\varphi(u)=(u+1)^2/2$. Notice that, in this case, the optimal decision function is given by: $\forall x\in \mathcal{X}$, $h^*(x)=2\mathbb{P}\{Y=+1\mid X=x\}-1$. The classification rule $H^*(x)=\mathrm{sign}(h^*(x))$ thus coincides with the naive Bayes classifier. We abusively set $L_{\varphi}(\theta)=L_{\varphi}(h(.,\; \theta))$ for all $\theta \in \Theta$.
Consider the problem of finding a classification rule with minimum risk, \textit{i.e.} the optimization problem $\min_{\theta\in\Theta} L_\varphi(\theta)$.
In the ideal case where a standard gradient descent could be applied, one would iteratively generate a sequence
$\theta(t) = (\theta_1(t),\cdots,\theta_d(t))$, $t\geq 1$,
satisfying the following update equation:
$$
\theta(t+1) = \theta(t) + \gamma(t)\, \mathbb{E}\left[Y\nabla_{\theta} h(X,\theta(t))\varphi'(-YH(X,\theta(t)))\right],
$$
where $\gamma(t)>0$ is the learning rate. Naturally, as $(X,Y)$'s distribution is unknown, the expectation involved
in the $t$-th iteration cannot be computed and must be replaced by a statistical version, $$(1/N)\sum_{i=1}^N \{Y_i\nabla_{\theta} h(X_i,\theta(t))\varphi'(-Y_iH(X_i,\theta(t)))  \}$$ in accordance with the \textit{Empirical Risk Minimization} paradigm. This is a particular case of the problem previously described, where $Z=(X,Y)$ and $\psi(Z,\theta)=\varphi(-Yh(X,\theta))$.
\end{example}
\begin{example}\label{ex:logistic}{\sc (Logistic regression)} Consider the same probabilistic model as above, except that the goal pursued is to find $\theta\in \Theta$ so as to minimize
\begin{multline*}
-\sum_{i=1}^N\left\{\frac{Y_i+1}{2}\log \left(\frac{exp(h(X_i,\theta))}{1+exp(h(X_i,\theta))}\right)+\frac{1-Y_i}{2}\log\left(\frac{1}{1+exp(h(X_i,\theta))}\right)\right\},
\end{multline*}
which is nothing else than the opposite of the conditional log-likelihood given the $X_i$'s related to the parametric logistic regression model: $\theta\in \Theta$, $$\mathbb{P}_{\theta}\{Y=+1\mid X  \}=exp(h(X,\theta))/(1+exp(h(X,\theta))).$$
\end{example}
\subsection{Survey sampling}
Let $(\Omega, \mathcal{A},\mathbf{P})$ be a probability space and $N\geq 1$. In the framework we consider throughout the article, it is assumed that $Z_1,\;\ldots,\; Z_N$ is a sample of i.i.d. random variables defined on $(\Omega, \mathcal{A},\mathbf{P})$, taking their values in $\mathbb{R}^d$. The $Z_i$'s correspond to independent copies of a generic r. v. $Z$ observed on a finite population $\mathcal{U}_N := \{1,\; \ldots,\; N\}$.
We call a \textit{survey sample} of (possibly random) size $n \leq N$ of the population $\mathcal{U}_N$, any subset $s := \{i_1, \dots, i_{n(s)}\} \in \mathcal{P}(\mathcal{U}_N)$ with cardinality $n =: n(s)$ less that $N$. Given the statistical population $\mathcal{U}_N$, a sampling scheme (design/plan) without replacement is determined by a probability distribution $R_N$ on the set of all possible samples $s \in \mathcal{P}(\mathcal{U}_N)$. For any $i\in \{1, \dots, N\}$, the (first order) \textit{inclusion probability}, 
\begin{equation*} 
\pi_i(R_N) := \mathbb{P}_{R_N}\{i \in S\},
\end{equation*}
is the probability that the unit $i$ belongs to a random sample $S$ drawn from distribution $R_N$. We set $\boldsymbol{\pi}(R_N) := (\pi_1(R_N), \dots, \pi_N(R_N))$. The second order inclusion probabilities are denoted by
\begin{equation*}
\pi_{i,j}(R_N) := \mathbb{P}_{R_N}\{(i,j)\in S^2\},
\end{equation*}
for any $(i,j)$ in $\{1,\dots,N\}^2$. Equipped with these notation, we have $\pi_{i,i}=\pi_i$ for $1\leq i\leq N$. When no confusion is possible, we shall omit to mention the dependence in $R_N$ when writing the first/second order probabilities of inclusion. The information related to the resulting random sample $S \subset \{1,\dots, N\}$ is fully enclosed in the r.v. $\boldsymbol{\epsilon}_N := (\epsilon_1,\dots,\epsilon_N)$, where $\epsilon_{i} = \mathbb{I}\{i\in S\}$.
Given the statistical population, the conditional $1$-d marginal distributions of the sampling scheme $\boldsymbol{\epsilon}_N$ are the Bernoulli distributions $\mathcal{B}(\pi_{i})=\pi_i \delta_1 +(1-\pi_i)\delta_0$, $1\leq i\leq N$, and the conditional covariance matrix of the r.v. $\boldsymbol{\epsilon}_N$ is given by $\Gamma_{N} := \left\{\pi_{i,j} - \pi_{i}\pi_{j}  \right\}_{1\leq i, j\leq N}$.
Observe that, equipped with the notations above, $\sum_{i=1}^{N}\epsilon _{i} = n(S)$.

\par One of the simplest survey plans is the Poisson scheme (without replacement). For such a plan $T_N$, conditioned upon the statistical population of interest, the $\epsilon_i$'s are independent Bernoulli random variables with parameters $p_1,\;\ldots,\; p_N$ in $(0,1)$. The first order inclusion probabilities thus characterize fully such a plan: equipped with the notations above, $\pi_i(T_N)=p_i$ for $i\in\{1,\; \ldots,\; N \}$. Observe in addition that the size $n(S)$ of a sample generated this way is random with mean $\sum_{i=1}^Np_i$ and goes to infinity as $N\rightarrow +\infty$ with probability one, provided that $\min_{1\leq i \leq N}p_i$ remains bounded away from zero. In addition to its simplicity (regarding the procedure to select a sample thus distributed), it plays a crucial role in sampling theory, insofar as it can be used to build a wide range of survey plans by conditioning arguments, see \cite{Hajek64}. For instance, a \textit{rejective sampling plan} of size $n\leq N$ corresponds to the distribution of a Poisson scheme $\boldsymbol{\epsilon}_N$ conditioned upon the event $\{\sum_{i=1}^N\epsilon_i=n  \}$. One may refer to \cite{Cochran77}, \cite{Dev87} for accounts of survey sampling techniques and examples of designs to which the subsequent analysis applies.

\subsection{The Horvitz-Thompson estimator}
Suppose that independent r.v.'s $Q_1,\; \ldots,\; Q_N$, copies of a generic variable $Q$ taking its values in $\mathbb{R}^d$, are observed on the population $\mathcal{U}_N$.
A natural approach to estimate the total $\mathbf{Q}_N=\sum_{i=1}^N Q_i$ based on a sample $S\subset\{1,\; \ldots,\; N  \}$ generated from a survey design $R_N$ with (first order) inclusion probabilities $\{\pi_i\}_{1\leq i \leq N}$ consists in computing the {\it Horvitz-Thompson estimator} (HT estimator in abbreviated form)
\begin{equation}\label{eq:HT_total}
\bar{\mathbf{Q}}^{HT}_{R_N}=\sum_{i\in S}\frac{1}{\pi_i}Q_i=\sum_{i=1}^N\frac{\epsilon_i}{\pi_i}Q_i,
\end{equation}
with $0/0=0$ by convention. Notice that, given the whole statistical population $Q_1,\; \ldots,\; Q_N$, the HT estimator is an unbiased estimate of the total: $\mathbb{E}[ \bar{\mathbf{Q}}^{HT}_{R_N}\mid Q_1,\; \ldots,\; Q_N ]=\mathbf{Q}_N$ almost-surely. When samples drawn from $R_N$ are of fixed size, the conditional variance is given by:
\begin{equation}\label{eq:var_cond}
var\left( \bar{\mathbf{Q}}^{HT}_{R_N}\mid Q_1,\; \ldots,\; Q_N \right)=\sum_{i< j}\Vert\frac{Q_i}{\pi_i}-\frac{Q_j}{\pi_j} \Vert^2(\pi_{i,j}-\pi_i\pi_j).
\end{equation}
When the survey design is a Poisson plan $T_N$ with probabilities $p_1,\; \ldots,\; p_N$, it is given by:
\begin{equation}\label{eq:var_Poisson}
var\left( \bar{\mathbf{Q}}^{HT}_{T_N}\mid Q_1,\; \ldots,\; Q_N \right)=\sum_{i=1}^N\frac{1-p_i}{p_i}\Vert Q_i
\Vert ^2.
\end{equation}

\begin{remark}\label{rk:aux}{\sc (Auxiliary information)}
In practice, the first order inclusion probabilities are defined as a function of an \textit{auxiliary variable}, $W$ taking its values in $\mathbb{R}^{d'}$ say, which is observed on the entire population (\textit{e.g.} a $d'$-dimensional marginal vector $Z'$ for instance): for all $i \in \{1, \dots, N\}$ we can write $\pi_i = \pi(W_i)$ for some link function $\pi:\mathbb{R}^{d'}\rightarrow (0,1)$. When $W$ and $Q$ are strongly correlated, proceeding this way may help us select more informative samples and consequently yield estimators with reduced variance. A more detailed discussion on the use of auxiliary information in the present context can be found in subsection \ref{subsec:optimal}.
\end{remark}

Going back to the SGD problem, the {\it Horvitz-Thompson estimator} of the gradient $\widehat{l}_N(\theta)$ based on a survey sample $S$ drawn from a design $R_N$ with (first order) inclusion probabilities $\{\pi_i \}_{1\leq i \leq N}$ is:
\begin{equation}\label{eq:HT_gradient}
\overline{l}^{HT}_{\pi}(\theta)=\frac{1}{N}\sum_{i\in S}\frac{1}{\pi_i}\nabla_{\theta}\psi(Z_i,\theta).
\end{equation}
As pointed out in Remark \ref{rk:aux}, ideally, the quantity $\pi_i$ should be strongly correlated with $\nabla_{\theta}\psi(Z_i,\theta)$. Hence, this leads to consider a procedure where the survey design used to estimate the gradient may change at each step, as in the {\sc HTGD} algorithm described in the next section. For instance, one could stipulate the availability of extra information taking the form of random fields on a space $\mathcal{W}$, $\{W_i(\theta)\}_{\theta\in \Theta}$ with $1\leq i \leq N$, and assume the existence of a link function $\pi: \mathcal{W}\rightarrow (0,1)$ such that $\pi_i=\pi(W_i(\theta))$.
Of course, such an approach is of benefit only when the cost of the computation of the weight $\pi(W_i(\theta))$ is smaller than that of the gradient $\nabla_{\theta}\psi(Z_i,\theta)$. As shall be seen in section \ref{sec:num}, this happens to be the case in many situations encountered in practice.

\section{Horvitz-Thompson gradient descent}\label{sec:HTalgo}

This section presents, in full generality, the variant of the SGD method we promote in this article. It can be implemented in particular when some extra information about the target (the gradient vector field in the present case) is available, allowing hopefully for picking a sample yielding a more accurate estimation of the (true) gradient than that obtained by means of a sample chosen completely at random. Several tuning parameters must be picked by the user, including the parameter $N_0$ which controls the number of terms involved in the empirical gradient estimation at each iteration, see Fig. \ref{HTGDalgo}.

\begin{figure}[h!]
\begin{center}
\fbox{
\begin{minipage}[t]{13cm}
\medskip

\begin{center}
{\sc Horvitz-Thompson Gradient Descent Algorithm (HTGD)}
\end{center}
\medskip
{\small
\begin{enumerate}
\item[] ({\sc Input.}) Datasets $\{Z_1,\; \ldots,\; Z_N  \}$ and $\{W_1,\; \ldots,\; W_{N}\}$. Maximum (expected) sample size $N_0\leq N$. Collection of sampling plans $R_N(\theta)$ with first order inclusion probabilities $\pi_i(\theta)$ for $1\leq i \leq N$, indexed by $\theta\in \Theta$ with (expected) sample sizes less than $N_0$. Learning rate $\gamma(t)>0$. Number of iterations $T\geq 1$.
\medskip

\item ({\sc Initialization.}) Choose $\widehat{\theta}(0)$ in $\Theta$.
\medskip

\item ({\sc Iterations.}) For $t=0,\; \ldots,\; T$
\begin{enumerate}
\item Draw a survey sample $S=S_t$, described by $\boldsymbol{\epsilon}_N^{(t)}=(\epsilon_1^{(t)},\; \ldots,\; \epsilon_N^{(t)})$ according to $R_N(\widehat{\theta}(t))$ with inclusion probabilities $\pi_i(\widehat{\theta}(t))$ for $i=1,\; \ldots,\; N$.
\item Compute the HT gradient estimate at $\widehat{\theta}(t)$  
$$
\bar{l}^{HT}_{\pi}(\widehat{\theta}(t))\overset{def}{=}\frac{1}{N}\sum_{i=1}^N\frac{\epsilon_i^{(t)}}{\pi_i(\widehat{\theta}(t))}\nabla_{\theta}\psi(Z_i,\widehat{\theta}(t)).
$$
\item Update the estimator
$$
\widehat{\theta}(t+1)=\widehat{\theta}(t)-\gamma(t)\, \bar{l}^{HT}_{\pi}(\widehat{\theta}(t)). 
$$
\end{enumerate}
\medskip

\item[] ({\sc Output.})  The {\sc HTGD} estimator $\widehat{\theta}(T)$.
\end{enumerate}
\medskip
}
\end{minipage}
}
\end{center}

\caption{The generic {\sc HTGD} algorithm}\label{HTGDalgo}
\end{figure}
\bigskip

\par The asymptotic accuracy of the estimator or decision rule produced by the algorithm above as $T\rightarrow +\infty$ is investigated in the next section under specific assumptions.

\begin{remark}\label{rk:tradeoff}{\sc (Balance between accuracy and computational cost)} We point out that the complexity of any Poisson sampling algorithm is $O(N)$, just as in the usual case where data involved in the standard SGD are uniformly drawn with(out) replacement. However, even if it can be straightforwardly parallelized, the numerical computation of the inclusion probabilities at each step naturally induces a certain amount of additional latency. Hence, although HTGD may largely outperform SGD for a fixed number of iterations, this should be taken into consideration for optimizing computation time.

\end{remark}
\newpage
\section{Main results}\label{sec:main}
This section is dedicated to analyze the performance of the {\sc HTGD} method under adequate constraints, related to the (expected) size of the survey samples considered. We first focus on Poisson survey schemes and next discuss how to establish results in a general framework.

\subsection{Poisson schemes}\label{subsec:optimal}
Fix $\theta\in \Theta$ and $\mu_N\in (0,N)$. Given the sample $Z_1,\; \ldots,\; Z_N$, consider a Poisson scheme with parameter $p=(p_1,\; \ldots,\; p_N)$. In this case, Eq. \eqref{eq:var_cond} becomes:
\begin{equation*}
\mathbb{E}\left[\vert\vert \bar{l}^{HT}(\theta)-\widehat{l}_N(\theta) \vert\vert^2 \mid Z_1,\; \ldots,\; Z_N\right]=
\frac{1}{N^2}\sum_{i=1}^N \frac{1-p_i}{p_i}\vert\vert\nabla_{\theta}\psi(Z_i,\theta)  \vert\vert^2.
\end{equation*}
Searching for the parameters $p_1, \; \ldots,\; p_N$ such that the $L_2$ distance between the empirical gradient evaluated at $\theta$ and the HT version given $Z_1,\; \ldots,\; Z_N$ is minimum under the constraint that the expected sample size is equal to $N_0\in [0,N]$ yields the optimization problem:
\begin{equation}\label{eq:Poisson_opt}
\min_{p\in (0,1)^N} \sum_{i=1}^N \frac{1-p_i}{p_i}\vert\vert\nabla_{\theta}\psi(Z_i,\theta)  \vert\vert^2 \text{ s.t. } \sum_{i=1}^Np_i=N_0.
\end{equation}
As can be shown by means of the Lagrange multipliers method, the solution corresponds to weights being proportional to the values taken by the norm of the gradient:
\begin{equation}
\widetilde{p}_i(\theta)\overset{def}{=}N_0 \frac{\vert\vert\nabla_{\theta}\psi(Z_i,\theta)  \vert\vert}{\sum_{j=1}^N \vert\vert\nabla_{\theta}\psi(Z_j,\theta)  \vert\vert}.
\end{equation}
However, selecting a sample distributed this way requires to know the full statistical population $\nabla_{\theta}\psi(Z_i,\theta)$. In practice, one may consider situations where the weights are defined by means of a link function $\pi(W,\theta)$ and auxiliary variables $W_1,\; \ldots,\; W_N$ correlated with the $Z_i$'s, as suggested previously. Observe in addition that the goal pursued here is not to estimate the gradient but to implement a stochastic gradient descent involving an expected number of terms fixed in advance, while yielding results close to those that would be obtained by means of a gradient descent algorithm with mean field $(1/N)\sum_{i=1}^N\nabla_{\theta}\psi(Z_i,\theta)$ based on the whole dataset. However, as shall be seen in the subsequent analysis, in general these two problems do not share the same solution from the angle embraced in this article.

In the next subsection, assumptions on the survey design under which the {\sc HTGD} method yields accurate asymptotic results, surpassing those obtained with equal inclusion probabilities (\textit{i.e.} $p_i=N_0/N$ for all $i\in\{1,\;\ldots,\; N \}$), are exhibited.

\subsection{Limit theorems}

We now consider a collection of general (\textit{i.e.} not necessarily Poisson) sampling designs $\{R_N(\theta)\}_{\theta\in \Theta}$ and investigate the limit properties of the $M$-estimator produced by the {\sc HTGD} algorithm conditioned upon the data $\mathcal{D}_N=\{Z_1,\; \ldots,\; Z_N \}$ (or $\mathcal{D}_N=\{(Z_1,W_1),\; \ldots,\; (W_N,Z_N) \}$ in presence of extra variables, \textit{cf} Remark \ref{rk:aux}). The asymptotic analysis involves the \textit{regularity conditions} listed below, which are classically required in stochastic approximation.
\begin{assumption}\label{assumption_1} The conditions below hold true.
  \begin{itemize}
  \item For any $z$, $\theta\mapsto \psi(z,\theta)$ is of class $\mathcal{C}^1$.
  \item For any compact set $\mathcal{K}\subset\mathbb{R}^{d}$, we have with probability one: $\forall i\in\{1,\; \ldots,\; N\}$,
   \begin{equation*}
    \sup_{\theta\in \mathcal{K}}\frac{\left\|\nabla_{\theta} \psi(Z_i,\theta)\right\|}{\pi_i(\theta)}<+\infty.
    \end{equation*}
\item The set of stationary points $\mathcal{L}_N = \{\theta:\nabla_{\theta} \widehat{L}_N(\theta)=0\}$ is of finite cardinality.
  \end{itemize}
\end{assumption}

\begin{theorem}\label{thm:consistency} {\sc (Consistency)} Assume that the learning rate decays to $0$ so that $\sum_{t\geq 1} \gamma(t)=+\infty$ and $\sum_{t\geq 0}\gamma^2(t)<+\infty$. Suppose also that the {\sc HTGD} algorithm is stable, \textit{i.e.} there exists a compact set $\mathcal{K}\subset \mathbb{R}^d$ s.t. $\theta(t)\in \mathcal{K}$ for all $t\geq 0$. Under Assumption \ref{assumption_1}, conditioned upon the data $\mathcal{D}_N$, the sequence $\{\widehat{\theta}(t)  \}_{t\geq 0}$ converges to an element of the set $\mathcal{L}_N$ with probability one, as $t\rightarrow +\infty$.
\end{theorem}
The stability condition is generally difficult to check. In practice, one may guarantee it by confining the sequence to a compact set fixed in advance and using a \textit{projected} version of the algorithm above. For simplicity, the present study is restricted to the simplest framework for stochastic gradient descent and we refer to \cite{KYBook} or \cite{Borkar} (see section 5.4 therein) for further details.

Consider $\theta^*\in \mathcal{L}$. The following \textit{local} assumptions are also required to establish asymptotic normality results conditioned upon the event $\mathcal{E}(\theta^*)=\{ \lim_{t\rightarrow +\infty}\widehat{\theta}(t)= \theta^*\}$.
\begin{assumption}\label{assumption_2} The conditions below hold true.
  \begin{itemize}
  \item There exists a neighborhood $\mathcal{V}$ of $\theta^*$ such that for all $z$, the mapping $\theta\mapsto \psi(z,\theta)$ is of class $\mathcal{C}^2$ on $\mathcal{V}$.
  \item The Hessian matrix $H=\nabla^2_{\theta}\widehat{L}_N(\theta^*)$ is a stable $q\times q$  positive-definite matrix: its smallest  eigenvalue is $l$ with $l>0$.
  \item For all $(i,j)\in\{1,\; \ldots,\; N  \}^2$, the mapping $\theta\in \mathcal{V}\mapsto \pi_{i,j}(\theta)$ is continuous.
  \end{itemize}
\end{assumption}
\begin{theorem} \label{thm:CLT}{\sc (A conditional CLT)} Suppose that Assumptions \ref{assumption_1}-\ref{assumption_2} are fulfilled and that $\gamma(t)=\gamma(0) t^{-\alpha}$ for some constants $\gamma(0)>0$ and $\alpha\in (1/2,1]$. When $\alpha=1$, take $\gamma(0)>1/(2l)$ and $\eta=1/(2\gamma(0))$. Set $\eta=0$ otherwise. Given the observations $Z_1,\; \ldots,\; Z_N$ (respectively, $(Z_1,W_1),\; \ldots,\; (Z_N,W_N)$) and conditioned upon the event $\mathcal{E}(\theta^*)$, we have the convergence in distribution as $t\rightarrow +\infty$
$$
\sqrt{1/\gamma(t)}\left(\widehat{\theta}(t)-\theta^*  \right)\Rightarrow \mathcal{N}(0, \Sigma_{\pi}),
$$
where the asymptotic covariance matrix $\Sigma_{\pi}$ is the unique solution of the Lyapunov equation:
\begin{equation}\label{eq:Lyapounov}
H\Sigma +\Sigma H+2\eta\Sigma=\Gamma^*,
\end{equation}
with
\begin{multline}\label{eq:asympt_var}
\Gamma^*=\frac{1}{N^2}\sum_{i=1}^N \frac{1-\pi_i(\theta^*)}{\pi_i(\theta^*)}\nabla_{\theta}\psi(Z_i,\theta^*) \nabla_{\theta}\psi(Z_i,\theta^*)^T\\
+\frac{1}{N^2}\sum_{i\neq j}\frac{\pi_{i,j}(\theta^*)}{\pi_{i}(\theta^*)\pi_{j}(\theta^*)}\nabla_{\theta}\psi(Z_i,\theta^*) \nabla_{\theta}\psi(Z_j,\theta^*)^T.
\end{multline}
\end{theorem}

The result stated below provides the asymptotic conditional distribution of the error. Its proof is a direct application of the second order delta method and is left to the reader.
\begin{corollary}{\sc (Error rate)} Under the hypotheses of Theorem \ref{thm:CLT}, given the observations $Z_1,\; \ldots,\; Z_N$ (respectively, $(Z_1,W_1),\; \ldots,\; (Z_N,W_N)$) and conditioned upon the event $\mathcal{E}(\theta^*)$, we have the convergence in distribution towards a non-central chi-square distribution:
$$
1/\gamma(t)\left(\widehat{L}_N(\widehat{\theta}(t))-\widehat{L}_N(\theta^*)  \right)\Rightarrow \frac{1}{2}U^T \Sigma_{\pi}^{1/2}H\Sigma_{\pi}^{1/2}U,
$$
as $t\rightarrow +\infty$, where $U$ is a $d$-dimensional Gaussian centered r.v. with the identity as covariance matrix.
\end{corollary}

Before showing how the results above can be used to understand how specific sampling designs may improve statistical analysis, a few comments are in order.

\begin{remark}{(Asymptotic covariance estimation)} An estimate of $\Sigma_{\pi}$ could be obtained by solving the equation $\Sigma H+H\Sigma +2\eta\Sigma=\Gamma(\widehat{\theta}(T))$, replacing in \eqref{eq:asympt_var} the (unknown) target value $\theta^*$ by the estimate produced by the {\sc HTGD} algorithm after $T$ iterations. Alternatively, a percentile Bootstrap method could be also used for this purpose, repeating $B\geq 1$ times the {\sc HTGD} algorithm based on replicates of the original sample $\mathcal{D}_N$.
\end{remark}

For completeness, a rate bound analysis of the {\sc HTGD} algorithm is also provided in the Appendix section.

\subsection{Asymptotic covariance optimization in the Poisson case}\label{subsec:opt}

Now that the limit behavior of the solution produced by the {\sc HTGD} algorithm has been described for general collections of survey designs $\mathcal{R}=\{R_N(\theta)\}_{\theta\in \Theta}$ of fixed expected sample size, we turn to the problem of finding survey plans yielding estimates with minimum variability. Formulating this objective in a quantitative manner, this boils down to finding $\mathcal{R}$ so as to minimize $\vert\vert \Sigma_{\pi}^{1/2}\vert\vert$, for an appropriately chosen norm $\vert\vert .\vert\vert$ on the space $\mathcal{M}_q(\mathbb{R})$ of $q\times q$ matrices with real entries for instance. In order to get a natural summary of the asymptotic variability, we consider here the Hilbert-Schmidt norm, \textit{i.e.} $\vert\vert A\vert\vert_{HS}=\sqrt{Tr(AA^T)}=(\sum_{i,j}a_{i,j}^2)^{1/2}$ for any $A=(a_{i,j})\in \mathcal{M}_d(\mathbb{R})$ where $Tr(.)$ denotes the Trace operator. For simplicity's sake, we focus on Poisson schemes and consider the case where $\eta=0$. Let $\mathcal{P}=\{\mathbf{p}(\theta)=(p_1(\theta),\; \ldots,\; p_N(\theta) \}_{\theta\in \Theta}$ be a collection of first order inclusion probabilities.
The following result exhibits an optimal collection of Poisson schemes among those with $N_0$ as expected sizes, in the sense that it yields an HTGD estimator with an asymptotic covariance of square root with minimum Hilbert-Schmidt norm. We point out that it is generally different from that considered in subsection \ref{subsec:optimal}, revealing the difference between the issue of estimating the empirically gradient accurately by means of a Poisson Scheme and that of optimizing the HTGD procedure.
\begin{proposition}\label{prop:optimal}{\sc (Optimality)}
Let $Q=H^{-1/2}$. The collection $\mathbf{p}^*$ of Poisson designs defined by: $\forall i\in\{1,\;\ldots,\; N  \}$, $\forall \theta\in \Theta$,
$$
p_i^*(\theta)=N_0\frac{\vert\vert Q\nabla_{\theta}\psi(Z_i,\theta)\vert\vert }{\sum_{j=1}^N \vert\vert Q\nabla_{\theta}\psi(Z_j,\theta)\vert\vert}
$$
is a solution of the minimization problem
$$
\min_{\mathbf{p}}\vert\vert \Sigma^{1/2}_{\mathbf{p}}\vert\vert_{HS} \text{ subject to } \sum_{i=1}^Np_i(\theta)=N_0 \text{ for all } \theta \in \Theta,
$$
where the infimum is taken over all collections $\mathbf{p}$ of Poisson designs. In addition, we have
\begin{multline*}
2\vert\vert \Sigma_{\mathbf{p}^*}^{1/2}\vert\vert^2_{HS}=\frac{1}{N_0}\left(\frac{1}{N}\sum_{i=1}^N\vert\vert Q\nabla_{\theta}\psi(Z_i,\theta^*)\vert\vert\right)^2 \\
+\frac{2}{N^2} \sum_{i<j} (\nabla_{\theta}\psi(Z_i,\theta^*))^T H^{-1} \nabla_{\theta} \psi (Z_j,\theta^*) .
\end{multline*}
\end{proposition} 

Of course, the optimal solution exhibited in the result stated above is completely useless from a practical perspective, since the matrix $H$ is unknown in practice and the computation of the values taken by the gradient at each point $Z_i$ is precisely what we are trying to avoid in order to reduce the computational cost of the SGD procedure. In the next section, we show that choosing inclusion probabilities positively correlated with the $p^*_i(\theta)$'s is actually sufficient to reduce asymptotic variability (compared to the situation where equal inclusion probabilities are used). In addition, as illustrated by the two easily generalizable examples described in section \ref{sec:num}, such a sampling strategy can be implemented in many situations.

\subsection{Extensions to more general Poisson survey designs}\label{subsec:extension}
In this subsection, we still consider Poisson schemes and the case $\eta=0$ for simplicity and now place ourselves in the situation where the information at disposal consists of a collection of i.i.d. random pairs $(Z_1,W_1),\; \ldots,\; (Z_N,W_N)$ valued in $\mathbb{R}^d\times \mathbb{R}^{d'}$. We consider inclusion probabilities $$p_i(\theta)=N_0\frac{ p(W_i,\theta)}{\sum_{j=1}^Np(W_j,\theta)}$$ defined through a \textit{link function} $p:\mathbb{R}^{d'}\times \Theta\rightarrow (0,1)$, see Remark \ref{rk:aux}. The computational cost of $p(W_i,\theta)$ is assumed to be much smaller than that of $\nabla_{\theta}\psi(Z_i,\theta)$ (see the examples in section \ref{sec:num} below) for all $(i,\theta)\in \{1,\; \ldots,\; N  \}\times \Theta$. The assumption introduced below involves the empirical covariance $c_N(\theta)$ between $\vert\vert Q\nabla_{\theta}\psi(Z_,\theta)\vert\vert^{2}/p(W,\theta)$ and $p(W,\theta)$, for $\theta\in \Theta$. Observe that it can be written as:
\begin{multline*}
c_N(\theta)=\frac{1}{N}\sum_{i=1}^N \vert\vert Q\nabla_{\theta}\psi(Z_i,\theta)\vert\vert^{2}
-\frac{1}{N^2}\sum_{i=1}^N \frac{\vert\vert Q\nabla_{\theta}\psi(Z_i,\theta)\vert\vert^{2}}{p(W_i,\theta)} \sum_{i=1}^N p(W_i,\theta),
\end{multline*} 
with $\theta\in \Theta$.
\begin{assumption}\label{assumption:corr}
The link function $p(w,\theta)$ fulfills the following condition: 
$$ c_N(\theta^*)>0.$$
\end{assumption}

The result stated below reveals to which extent sampling with inclusion probabilities defined by some appropriate link function may improve upon sampling with equal inclusion probabilities, $\bar{p}_i=N_0/N$ for $1\leq i \leq n$, when implementing stochastic gradient descent. Namely, the accuracy of the HTGD gets closer and closer to the optimum, as the empirical covariance $c_N(\theta^*)$ increases to its maximum. Notice that in the case where inclusion probabilities are all equal, we have $c_N\equiv 0$.
\begin{proposition}\label{prop:gain} Let $N_0$ be fixed. Suppose that the collection of Poisson designs $\mathbf{p}$ with expected sizes $N_0$ is defined by a link function $p(w,\theta)$ satisfying Assumption \ref{assumption:corr}. Then, we have
$$
\vert\vert \Sigma_{\mathbf{p}}^{1/2}\vert\vert_{HS}< \vert\vert \Sigma_{\bar{\mathbf{p}}}^{1/2}\vert\vert_{HS},
$$
as well as
$$
0\leq \vert\vert \Sigma_{\mathbf{p}}^{1/2} \vert\vert^2_{HS}-\vert\vert \Sigma_{\mathbf{p}^*}^{1/2}  \vert\vert^2_{HS}
= \frac{1}{2N_0}\left\{ \sigma^2_N(\theta^*)-c_N(\theta^*)  \right\},
$$
where $\sigma^2_N(\theta)$ denotes the empirical variance of the r.v. $\vert\vert \nabla_{\theta}\psi(Z,\theta)\vert\vert$.
\end{proposition}

As illustrated by the easily generalizable examples provided in the next section, one may generally find link functions fulfilling Assumption \ref{assumption:corr} without great effort, permitting to gain in accuracy from the implementation of the HTGD algorithm.

\section{Illustrative numerical experiments}\label{sec:num}
For illustration purpose, this section shows how the results previously established apply to two problems by means of simulation experiments. For both examples, the performance of the {\sc HTGD} algorithm is compared with that of a basic SGD strategy with the same (mean) sample size.

\subsection{Linear logistic regression}
Consider the linear logistic regression model corresponding to Example \ref{ex:logistic} with $\theta=(\alpha, \beta)\in \mathbb{R}\times \mathbb{R}^{d}$ and $h(x,\theta)=\alpha+ \beta^T x$ for all $x\in \mathbb{R}^d$. Let $X'$ be a low dimensional marginal vector of the input r.v. $X$, of dimension $d'<<d$ say, so that one may write $X=(X',X'')$ as well as $\beta=(\beta',\beta'')$ in a similar manner. The problem of fitting the parameter $\theta$ through conditional MLE corresponds to the case
\begin{equation*}
\psi((x,y),\theta)=-\log\left(\frac{e^{\alpha+\beta^Tx}(y+1)/2 +(1-y)/2}{1+e^{\alpha+\beta^Tx}}\right).
\end{equation*}
We propose to implement the HTGD with the link function $\widetilde{p}((x',y),\theta)=\vert\vert \nabla_{\theta}\psi'((X,Y),\theta) \vert\vert$, where 
\begin{equation*}
\psi'((x,y),\theta)=-\log\left(\frac{e^{\alpha+{\beta'}^Tx'}(y+1)/2 +(1-y)/2}{1+e^{\alpha+{\beta'}^Tx'})}\right).
\end{equation*}
In order to illustrate the advantages of the HTGD technique for logistic regression, we considered the toy numerical model with parameters  $d=11$ and $\theta= (\alpha,\; \beta_1,\dots,\beta_{10}) = (-9,0,3,-9,4,-9,15,0,-7,1,0)$, the $10$ input variables being independent, uniformly distributed on $(0,1)$. The maximum likelihood estimators of $\theta$ were computed using the HTGD and SGD (mini-batch) . In order to compare them, the same number of iterations was chosen in each situation and a learning rate proportionnal to $1/\sqrt{t}$ was considered.

 As a first go, we drew a single sample of size $N=5000$ on which the two algorithms were performed for $2000$ iterations. Two sub-sample sizes were considered : $n=10$ and $n=100$. As can be seen on Fig. \ref{fig:bigexb6}, while virtually equivalent in terms of computation time, thus taking a larger sample improves the efficiency of the HTGD. It also appears to reach a better level of precision in less steps than both competitors, a phenomenon that is consistent on all $11$ coordinates of $\theta$.
 
\begin{figure}%
\centering
 \parbox{2.5in}{\includegraphics[width =2.5in]{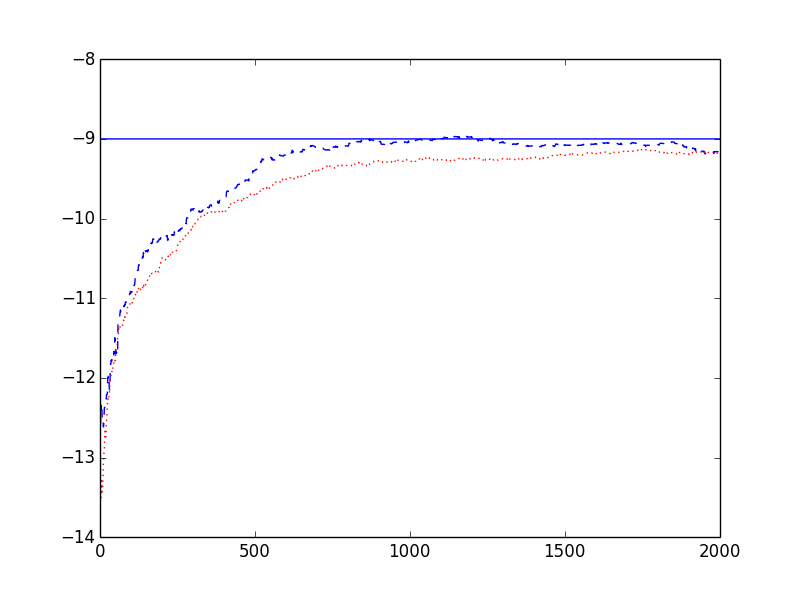}} \qquad
 \parbox{2.5in}{\includegraphics[width =2.5in]{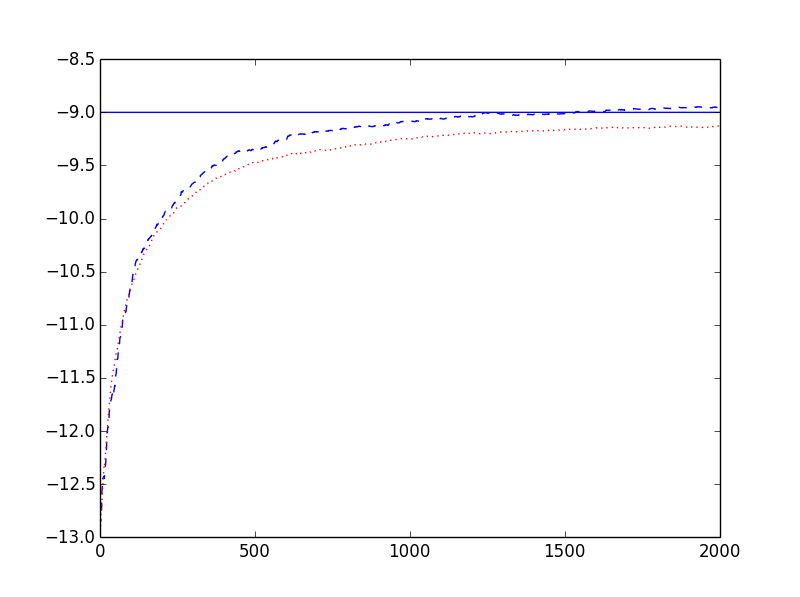}}
\caption{Evolution of the estimator of $\beta_5$ with the number of iterations in the HTGD (solid), mini-batch SGD (dotted) and GD (dashed) algorithms with $n = 10$ (left) and $n = 100$ (right)}
\label{fig:bigexb6}
     \end{figure}
     
So as to account for the randomness due to the data, we then simulated $50$ samples according to the model for two population sizes, $N = 500$ and $N = 1000$. For both the HTGD and the mini-batch SGD algorithms, a sub-sample size of $20$ was chosen. As shown in Table \ref{tb:sd}, the HTGD seems to be more robust to data randomness than SGD and GD. It is not surprising, since the sampling phase selects the most informative observations relative to the gradient descent, which makes HTGD less sensitive to the possible noise. It also provides more precise estimates, as suggested by the results in Table \ref{tb:avgvsspl}.

 \begin{table}[H]\begin{center}
 \begin{tabular}{l c c }
 \hline
  & $N = 500$ & $N = 1000$\\
 \hline
 HTGD & 1.52 & 1.45\\
 SGD & 2.21 & 2.09 \\

 \hline
 \end{tabular}
 \caption{Mean standard deviations of the final estimates of $\theta(=-9)$ across the $50$ simulations}
 \label{tb:sd}
 \end{center}
 \end{table}

 \begin{table}[H]\begin{center}
 \begin{tabular}{l c c c c c}
 \hline     
 & Min. & Median & Max. & Mean  & S.D.\\
 \hline
 \multicolumn{6}{c}{HTGD}\\
 \hline
 $\theta_5$ & -9.5 &  -8.7 & -7.8 & -8.6 & 1.45\\
 $\theta_6$ & 13.3  & 14.6 & 15.9 & 14.5 & 1.52\\
 \hline
  \multicolumn{6}{c}{SGD}\\
  \hline
  $\theta_5$ & -9.9&  -8.2 & -7.4 & -8.2 & 2.09\\
  $\theta_6$ & 12.7 & 13.9 & 16.6 & 15.2 & 2.21\\
 \end{tabular}
 
 \caption{Statistics on the global behavior of the final estimates of $\beta_5$ and $\beta_6$ across the $50$ simulations}
 \label{tb:avgvsspl}\end{center}
 \end{table}
 
 \begin{figure}%
\centering
 \parbox{5in}{\includegraphics[width =5in]{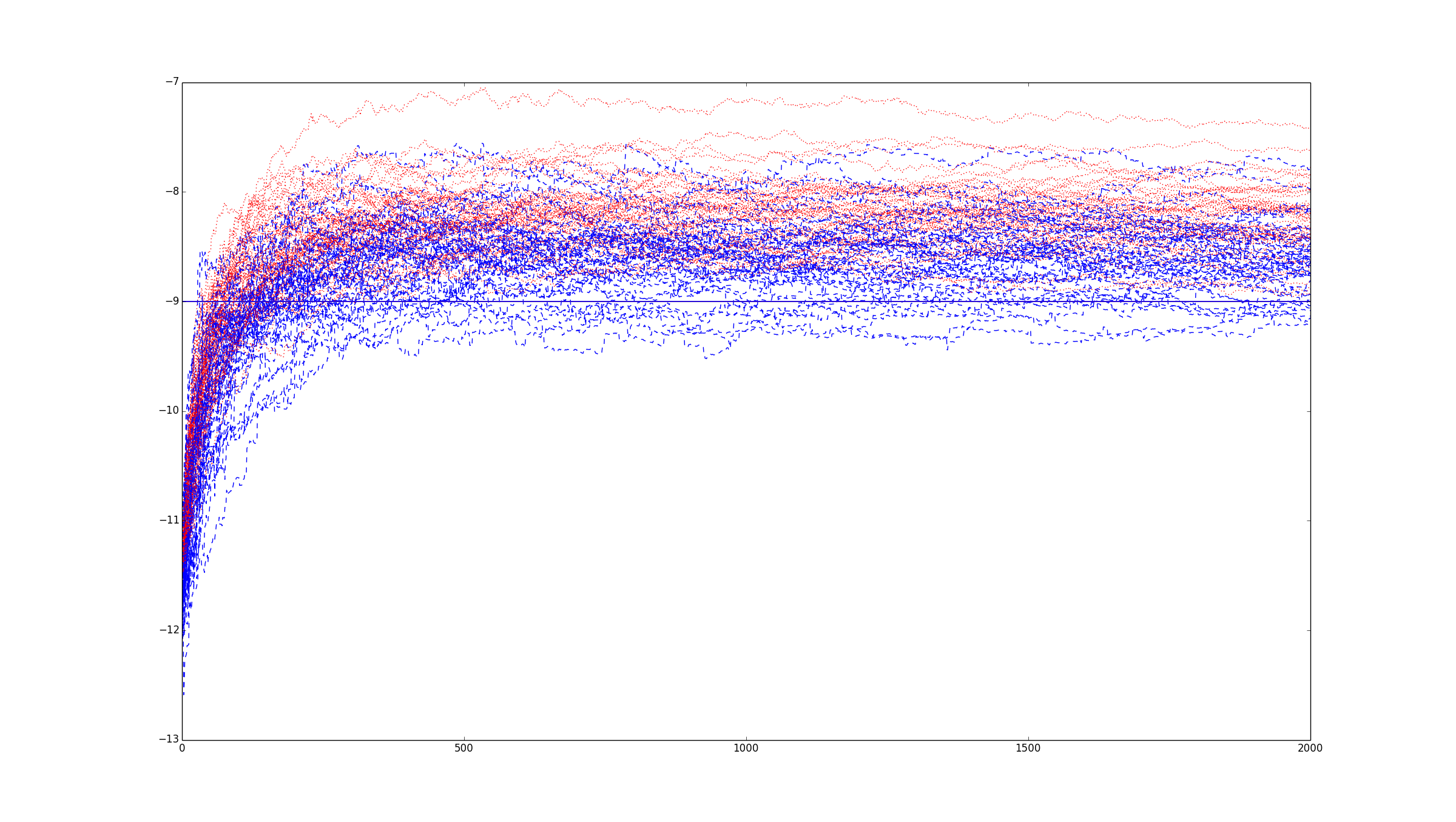}} \qquad
 \parbox{5in}{\includegraphics[width =5in]{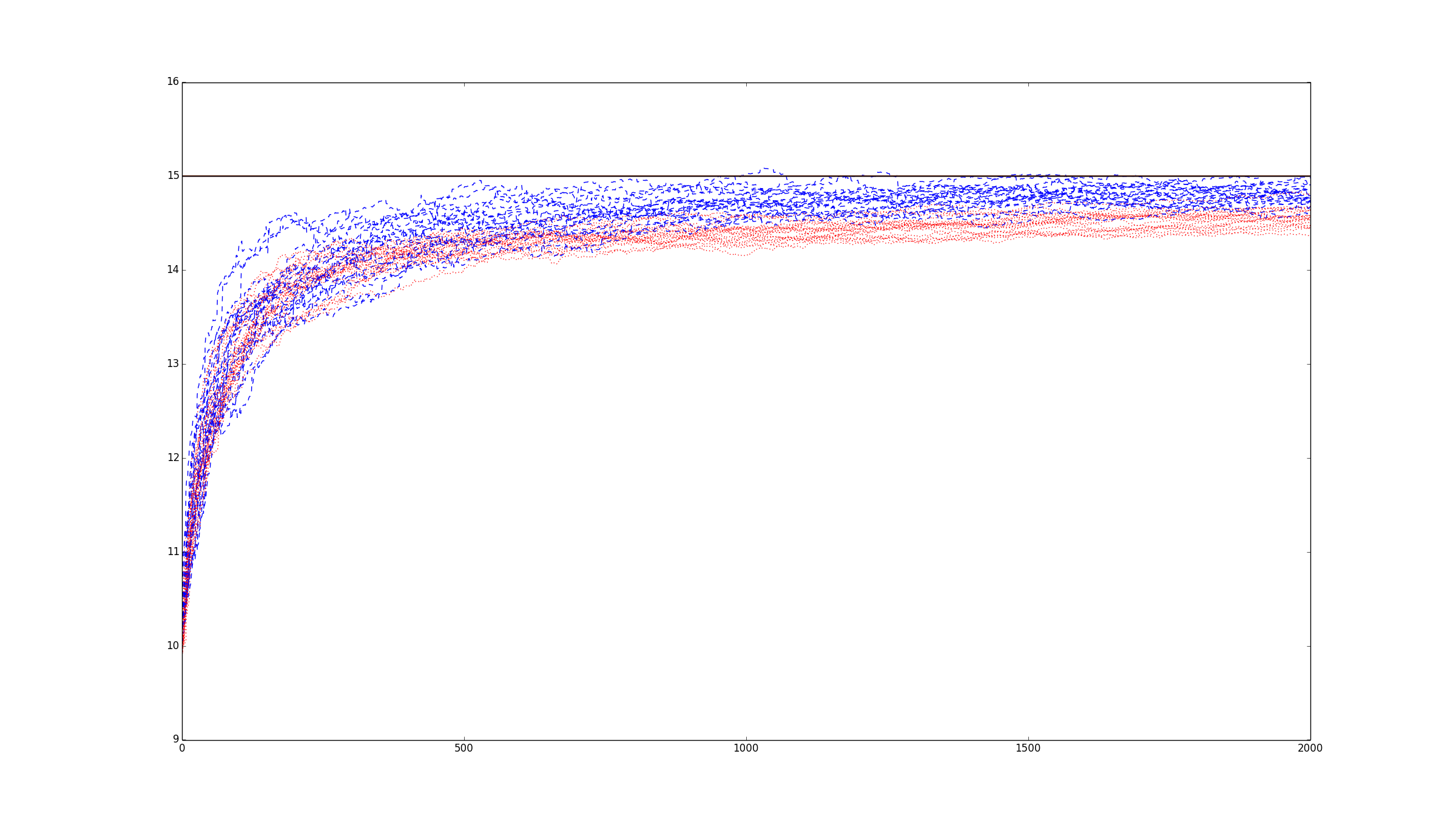}}
\caption{50 trajectories of the estimator of $\beta_5$ with the number of iterations in the HTGD (solid), mini-batch SGD (dotted) over 50 populations (left) and  of $\theta_6$ over 1 populations (right)}
\label{fig:bigexb6}
     \end{figure}

\subsection{The symmetric model} \label{ex:sym}

Consider now an i.i.d. sample $(X_{1},\; X_{2},\; \ldots,\; X_{N})$ drawn from an unknown probability distribution on $\mathbb{R}^d$, supposed to belong to the semi-parametric collection $\{P_{\theta ,f },\; \theta\in \Theta\} $, $\Theta\subset \mathbb{R}^d$,
dominated by some $\sigma$-finite measure $\lambda$. The related densities are denoted by $f(
x-\theta)$, where $\theta\in \Theta$ is a location parameter and a $f(x)$ a
(twice differentiable) density, symmetric about $0$, \textit{i.e.}$
f(x)=f(-x)$. The density $f$ is unknown in practice and may be multimodal. For simplicity, we assume here that $\Theta \subset \mathbb{R}$ but similar
arguments can be developed when $d>1$. For such a general semi-parametric model, it is well-known that neither the sample mean nor the median (if, for instance, the distribution does not weight the singleton $\{0\}$) are good candidates for estimating the location parameter $\theta$. In the
semiparametric literature this model is referred to as the \textit{symmetric model}, see \cite{BKRW93}. It is known that the tangent
space (\textit{i.e.} the set of scores) with respect to the parameter of interest $%
\theta $ and that with respect to the nuisance parameter are orthogonal.
The global tangent space at $P_{\theta ,f }$ is
given by 
\begin{equation*}
T_{L}\left[ P_{\theta ,f },\mathbb{P}\right] =\left\{ c\frac{f
^{\prime }\left( x-\theta \right) }{\eta \left( x-\theta \right) }%
+h(x-\theta );c\in \mathbb{R},\ h\in \dot{\mathbb{P}}_{2}\right\},  
\end{equation*}
where $\dot{\mathbb{P}}_{2}$ is the tangent space with respect to the
nuisance parameter:
\begin{equation*}
\dot{\mathbb{P}}_{2}=\left\{ h:\;\; \mathbb{E}_{P_{\theta ,f }}[h(X)]=0,\; h(x)=h(-x)\text{
and } \mathbb{E}_{P_{\theta ,f }}[h^{2}(X)]<\infty \right\}.
\end{equation*}%
Orthogonality simply results from the fact that $f'(x)$ is an odd function and implies that the parameter $\theta$ can be adaptively estimated, as if the density $f(x)$ was known, refer to \cite{BKRW93} for more details.
In practice $f(x)$ is estimated by means of some symmetrized kernel density
estimator. Given a Parzen-Rosenblatt kernel $K(x)$ (\textit{e.g.} a Gaussian kernel) for instance, consider the estimate
\[
\widetilde{f} _{\theta ,N}(x)=\frac{1}{Nh_{N}}\sum_{i=1}^{N}K\left( \frac{%
x-(X_{i}-\theta )}{h_{N}}\right),
\]
where $h_N>0$ is the smoothing bandwidth,
and form its symetrized version (which is an even function) 
\[
\widehat{f}_{\theta ,N}(x)=\frac{1}{2}\left( \widetilde{f} _{\theta ,N}(x) + \widetilde{f} _{\theta ,N}(-x) \right).
\]%
The related score is given by 
\[
\widehat{s}_{N}(x,\theta )=\frac{d}{d\theta }\widehat{f }_{\theta
,N}(x)/\widehat{f}_{\theta ,N}(x).
\]
In order to perform maximum likelihood estimation approximately, one can try to implement a gradient descent method to get an efficient estimator of $\theta$. For instance, for a reasonable sample size $%
N$, it is possible to show that, starting for instance from the empirical median $\theta _{0}$ with an adequate choice of the rate $\gamma _{t}$,
the sequence 
\[
\widehat{\theta}(t)=\widehat{\theta}(t-1)+\gamma _{t}\frac{1}{N}\sum_{j=1}^{N}\widehat{s}%
_{N}(X_{j}-\widehat{\theta}(t-1),\; \widehat{\theta}(t-1))
\]%
converges to the true MLE. The complexity of this algorithm is
typically of order $2T\times N^{2}$ if $T\geq 1$ is the number of iterations, due the tedious computations to evaluate the kernel density estimator (and its derivatives) at all points $X_{i}-\widehat{\theta}(t-1)$. It is thus relevant in this case
to try to reduce it by means of (Poisson) survey sampling. The iterations of such an
algorithm would be then of the form 
\begin{eqnarray*}
\widehat{\theta}(t) &=&\widehat{\theta}(t-1)+\gamma _{t}\frac{1}{N}\sum_{j=1}^{N}\frac{%
\varepsilon _{j}}{p _{j}}\widehat{s}%
_{N}(X_{j}-\widehat{\theta}(t-1),\; \widehat{\theta}(t-1)), \\
\sum_{j=1}^{N}p _{j} &=& n.
\end{eqnarray*}%
As shown in section \ref{subsec:opt}, the optimal choice would be
to choose $p _{j}$ proportional to $\vert \widehat{s}%
_{N}(X_{j}-\widehat{\theta}(t-1),\; \widehat{\theta}(t-1))\vert$ at the $t$-th iteration:
\begin{equation}\label{eq:sym_opt}
p^* _{j}\left(\widehat{\theta}(t-1)\right)=\frac{N_0\vert \widehat{s}%
_{N}(X_{j}-\widehat{\theta}(t-1),\; \widehat{\theta}(t-1))\vert }{\sum_{i=1}^{N}\vert \widehat{s}%
_{N}(X_{j}-\widehat{\theta}(t-1),\; \widehat{\theta}(t-1))\vert }.
\end{equation}
Unfortunately this is not possible because $s$ is unknown and replacing $s(x-\theta)$ by $\widehat{s}_{N}(x-\widehat{\theta}(t-1),\; \widehat{\theta}(t-1))$ in \eqref{eq:sym_opt} yields obvious computational difficulties. For this reason, we suggest to use the (much simpler) Poisson weights:
\[
p _{j}(\theta )=n\vert X_{j}-\theta \vert /\sum_{j=1}^{N}\vert X_{j}-\theta \vert.
\]

Fig. \ref{fig:Gauss_mixt} depicts the performance of the {\sc HTGD} algorithm when $\theta=0$ and $f(x)$ is a balanced mixture of two Gaussian densities with means $4$ and $-4$ respectively and same variance $\sigma^2=1$, compared to that of the usual SGD method. Based on a population sample of size $N=1000$, the {\sc HTGD} and SGD methods have been implemented with $n=10$ and $T=3000$ iterations, whereas $30$ iterations have been made for the basic GD procedure (with $n=N=1000$) so that the number of gradient computations is of the same order for each method. For each instance of the algorithms we took $\theta_0$ equal to the median of the population.

\begin{figure}%
\centering
 \parbox{5in}{\includegraphics[width =5in]{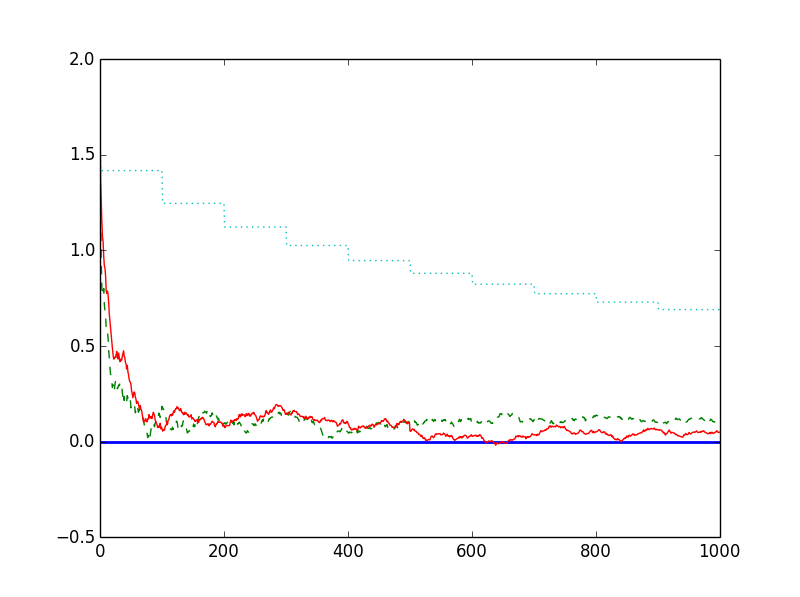}} \qquad
\caption{Evolution of the estimator of the location parameter $\theta=0$ of the balanced Gaussian mixture with the number of iterations in the HTGD (solid red), mini-batch SGD (dashed green) and GD (dotted blue) algorithms}
\label{fig:Gauss_mixt}
     \end{figure}

\begin{figure}%
\centering
 \parbox{5in}{\includegraphics[width =5in]{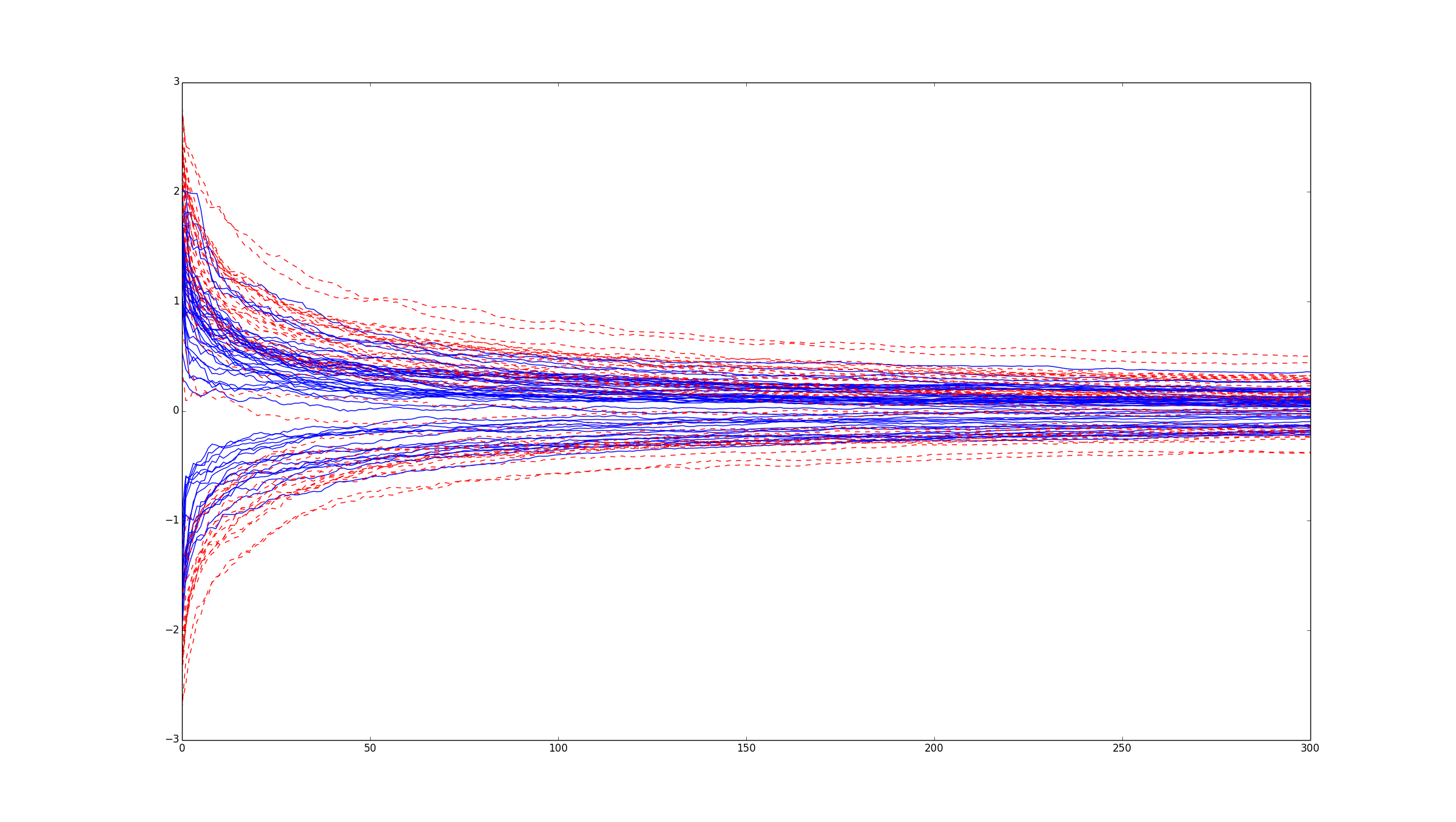}} \qquad
\caption{Evolution of the estimator of the location parameter $\theta=0$ of the balanced Gaussian mixture with the number of iterations in the HTGD (solid blue) and mini-batch SGD (dashed red) algorithms over 50 populations }
\label{fig:Gauss_mixt}
     \end{figure}
 \begin{table}[H]\begin{center}
 \begin{tabular}{l c c c c c}
 \hline     
 & Min. & Median & Max. & Mean  & S.D.\\
 \hline
 \multicolumn{6}{c}{HTGD}\\
 \hline
 $\theta$ & -0.35 &  0.006 & 0.29 & 0.014 & 0.16\\
 \hline
  \multicolumn{6}{c}{SGD}\\
  \hline
  $\theta$ & -0.38 &  -0.036 & 0.42 & 0.025 &0.22 \\
  \hline
  \multicolumn{6}{c}{GD}\\
  \hline
  $\theta$ & -0.52 &  -0.162 & 0.70 & 0.20 & 0.45\\
  \hline
 \end{tabular}
 
 \caption{Statistics on the global behavior of the final estimates of the location parameter across the $50$ simulations}
 \label{tb:avgvsspl}\end{center}
 \end{table}

\section{Conclusion}
Whereas massively parallelized/distributed approaches combined with random data splitting are now receiving much attention in the Big Data context, the present paper explores an alternative way of scaling up statistical learning methods, based on gradient descent techniques. It hopefully paves the way for incorporating efficiently survey techniques into machine-learning algorithms in order to exploit Big Data. Precisely, it shows how survey sampling can be used in order to improve the accuracy of the stochastic gradient descent method for a fixed number of iterations, while preserving the complexity of the procedure. Beyond theoretical limit results, the approach we promote is illustrated by promising numerical experiments.

\section*{Appendix A - Technical proofs}

\subsection*{Proof of Theorem \ref{thm:consistency}} 
Write the sequence as $$\widehat{\theta}(t+1)=\widehat{\theta}(t)-\gamma(t)\nabla_{\theta}\widehat{L}_N (\widehat{\theta}(t))+\gamma(t)\eta_{t+1},$$ where we set
$\eta_{t+1}=-\bar{l}^{HT}_{\pi}(\widehat{\theta}(t))+\nabla_{\theta}\widehat{L}_N (\widehat{\theta}(t))$, so that $-\nabla_{\theta}\widehat{L}_N (\widehat{\theta}(t))$ appears as the \textit{mean field} of the algorithm and $\eta_{t+1}$ as the \textit{noise term}. 
Consider the filtration $\mathcal{F}=\{\mathcal{F}_t\}_{t\geq 1}$ where $\mathcal{F}_t$ is the $\sigma$-field generated by $\boldsymbol{\epsilon}_1\; \ldots,\; \boldsymbol{\epsilon}_{t-1}$ for $t\geq 1$ and $Z_1,\; \ldots,\; Z_N$ (respectively $(Z_1,W_1),\; \ldots,\; (Z_N,W_N)$ in presence of extra information). We have $\mathbb{E}[\eta_{t+1}\mid \mathcal{F}_t ]=0$ for all $t\geq 1$, as well as:
\begin{multline*}
N^2\mathbb{E}[\vert\vert \eta_{t+1}\vert\vert^2 \mid \mathcal{F}_t]=\sum_{i=1}^N \frac{1-\pi_i(\widehat{\theta}(t))}{\pi_i(\widehat{\theta}(t))}\vert\vert  \nabla_{\theta}\psi(Z_i,\widehat{\theta}(t))\vert\vert^2+ \\\sum_{i\neq j}\left(\frac{\pi_{i,j}(\widehat{\theta}(t))}{\pi_i(\widehat{\theta}(t))\pi_j(\widehat{\theta}(t))}-1  \right)\nabla_{\theta}\psi(Z_i,\widehat{\theta}(t))^T \nabla_{\theta}\psi(Z_j,\widehat{\theta}(t))\\
\leq \sum_{i=1}^N\sup_{\theta \in \mathcal{K}} \frac{\vert\vert  \nabla_{\theta}\psi(Z_i,\theta)\vert\vert }{\pi_i(\theta)}\sup_{\theta \in \mathcal{K}}\vert\vert  \nabla_{\theta}\psi(Z_i,\theta)\vert\vert \\+\left( \sum_{i=1}^N\sup_{\theta \in \mathcal{K}} \frac{\vert\vert  \nabla_{\theta}\psi(Z_i,\theta)\vert\vert }{\pi_i(\theta)}\right)^2 <+\infty.
\end{multline*}
The consistency result thus holds true under the stipulated assumptions, see Theorem 2 in \cite{Delyon} or Theorem 2.2 in \cite{KYBook} for instance.
\subsection*{Proof of Theorem \ref{thm:CLT}} As observed in the preceding proof, $\{\eta_t\}_{t\geq 1}$ is a sequence of increments of a $d$-dimensional square integrable martingale adapted to the filtration $\mathcal{F}$. The proof is a direct application of Theorem 1 in \cite{Pelletier98} and consists in checking that the hypotheses of this result are fulfilled. Observe that the required conditions for the mean field hold true. Considering next the noise sequence of the algorithm, notice first that $\sup_{t\geq 0}\mathbb{E}[\vert\vert \eta_{t+1}\vert\vert^{b}\mid \mathcal{F}_t]\mathbb{I}\{\widehat{\theta}(t)\in \mathcal{V} \}<+\infty$ for any $b>2$. Indeed, we have
\begin{equation*}
\sup_{t\geq 0}\vert\vert \eta_{t+1}\vert\vert\mathbb{I}\{\widehat{\theta}(t)\in \mathcal{V} \}\leq 
\frac{2}{N}\sum_{i=1}^N \sup_{\theta\in \mathcal{V}}\frac{\vert\vert \nabla_{\theta}\psi(Z_i,\theta)\vert\vert}{\pi_i(\theta)}.
\end{equation*}
In addition, we have $\mathbb{E}[\eta_{t+1} \eta_{t+1}^T \mid \mathcal{F}_t]=\Gamma(\widehat{\theta}(t))$ for all $t\geq 1$, where: $\forall \theta\in \Theta$,
\begin{multline*}
\Gamma(\theta)=\frac{1}{N^2}\sum_{i=1}^N \frac{1-\pi_i(\theta)}{\pi_i(\theta)}\nabla_{\theta}\psi(Z_i,\theta) \nabla_{\theta}\psi(Z_i,\theta)^T\\
+\frac{1}{N^2}\sum_{i\neq j}\frac{\pi_{i,j}(\theta)}{\pi_{i}(\theta)\pi_{j}(\theta)}\nabla_{\theta}\psi(Z_i,\theta) \nabla_{\theta}\psi(Z_j,\theta)^T.
\end{multline*}
By virtue of the continuity assumptions, we can apply Lebesgue's Dominated Convergence Theorem : as $t\rightarrow +\infty$, $\Gamma(\widehat{\theta}(t))\rightarrow \Gamma^*=\Gamma(\theta^*)$ on the event $\mathcal{E}(\theta^*)$. This concludes the proof.
\subsection*{Proof of Proposition \ref{prop:optimal}}
Observe that, in the case where $\eta=0$, the Lyapunov equation \eqref{eq:Lyapounov} can be rewritten as
$$
\Sigma_{\mathbf{p}}+H^{-1}\Sigma_{\mathbf{p}}H=H^{-1}\Gamma^*.
$$ 
We thus have:

\begin{eqnarray*}
2\vert\vert \Sigma^{1/2}_{\mathbf{p}}\vert\vert_{HS}^2 &=& Tr( H^{-1}\Gamma^*)  = Tr(\mathbb{E}[(H^{-1}\bar{l}^{HT}_{p}(\theta^*)) (\bar{l}^{HT}_{p}(\theta^*))^T]\\ &=& \mathbb{E}[Tr((Q\bar{l}^{HT}_{p}(\theta^*)) (Q\bar{l}^{HT}_{p}(\theta^*))^T] = \mathbb{E}[\Vert Q\bar{l}^{HT}_{p}(\theta^*)\Vert^2] \\ &=& \mathbb{E}[ \Vert \frac{1}{N}\sum_{i=1}^N\frac{\epsilon_i}{p_i} (Q\nabla_{\theta}\psi(Z_i,\theta^*))\Vert^2] \\ &=&\frac{1}{N^2}\left(\sum_{i=1}^N\frac{\Vert Q \nabla_{\theta}\psi(Z_i,\theta^*)\Vert^2}{p_i}+2 \sum_{i<j} (Q \nabla_{\theta}\psi(Z_i,\theta^*))^T(Q \nabla_{\theta}\psi(Z_j,\theta^*))\right)
\end{eqnarray*}

The desired result can be then derived straightforwardly, by repeating the Lagrange multipliers argument used in subsection \ref{subsec:optimal}.

\subsection*{Proof of Proposition \ref{prop:gain}}
Using the last equality in the previous proof and $\bar{p}_i=N_0/N $ we have:
\begin{multline*}
2\{\vert\vert \Sigma^{1/2}_{\mathbf{\bar{p}}}\vert\vert_{HS}^2-\vert\vert \Sigma^{1/2}_{\mathbf{p}}\vert\vert_{HS}^2  \}=\frac{1}{N^2}\sum_{i=1}^N \frac{N}{N_0}\vert \vert Q\nabla_{\theta}\psi(Z_i,\theta^*) \vert\vert^2\\
-\frac{1}{N^2}\sum_{i=1}^N
\frac{\sum_{j=1}^Np(W_j,\theta^*)}{N_0p(W_i,\theta^*)} \vert \vert Q\nabla_{\theta}\psi(Z_i,\theta^*) \vert\vert^2
=c_N(\theta^*)/N_0.
\end{multline*}
This proves the first assertion. In addition, one has that
\begin{multline*}
0\leqslant 2N_0\{\vert\vert \Sigma^{1/2}_{\mathbf{p}}\vert\vert_{HS}^2-\vert\vert \Sigma^{1/2}_{\mathbf{p}^*}\vert\vert_{HS}^2  \}=2N_0\{\vert\vert \Sigma^{1/2}_{\mathbf{p}}\vert\vert_{HS}^2 -\vert\vert \Sigma^{1/2}_{\mathbf{\bar{p}}}\vert\vert_{HS}^2 + \vert\vert \Sigma^{1/2}_{\mathbf{\bar{p}}}\vert\vert_{HS}^2 \\
-\vert\vert \Sigma^{1/2}_{\mathbf{p}^*}\vert\vert_{HS}^2  \}
=\frac{1}{N^2}\sum_{i=1}^N\frac{1}{p(W_i,\theta^*)}\vert\vert Q\nabla_{\theta}\psi(Z_i,\theta^*)\vert\vert^2\times \sum_{i=1}^N p(W_i,\theta^*) \\-\left(\frac{1}{N}\sum_{i=1}^N\vert\vert Q\nabla_{\theta}\psi(Z_i,\theta^*) \vert\vert\right)^2
=\sigma^2_N(\theta^*)-c_N(\theta^*),  
\end{multline*}
which establishes the second assertion.

\section*{Appendix B - Rate Bound Analysis}

Here we establish a rate bound for the {\sc HTGD} algorithm under the assumption that the mapping $\theta\mapsto \psi(z,\theta)$ is convex, referred to as Assumption 4. Note that  assumptions 2. and 4.   implies that $\theta^*$ is unique and $\widehat{L}_N$ is $l$ strongly convex on $\mathcal{V}$. For simplicity's sake, we suppose that the strong convexity property holds true on $\mathbb{R}^d$. The following result relies on standard arguments in stochastic approximation, see \cite{Nemirovski:2009:RSA:1654243.1654247}, \cite{conf/nips/BachM11} or \cite{opac-b1104789}.

\begin{theorem}
Under Assumptions 1, 2 and 4 and for a stepsize $\gamma(t)=\gamma(0) t^{-\alpha}$ with some constants $\gamma(0)>0$ and $\alpha\in (1/2,1]$ (when $\alpha=1$, take $\gamma(0)>1/(2l)$), there exists a constant $\widetilde{C}_{\alpha}<+\infty$ such that: $\forall t\geq 1$,
\begin{equation}
\mathbb{E}[\Vert \widehat{\theta}(t)-\theta^*\Vert^2] \leq \dfrac{\widetilde{C}_{\alpha}}{t^{\alpha}}.
\end{equation}
\end{theorem}

\begin{proof}
We restrict ourselves to the case $\alpha=1$ and follow the proof of \cite{conf/nips/BachM11}. By construction, we have
\begin{equation*}
 \Vert \widehat{\theta}(t+1)-\theta^*\Vert ^2=\Vert \widehat{\theta}(t)-\theta^*\Vert ^2-2\gamma(t) \bar{l}^{HT}_{\pi}(\widehat{\theta}(t))^T(\widehat{\theta}(t)-\theta^*)+\Vert\gamma(t) \bar{l}^{HT}_{\pi}(\widehat{\theta}(t))\Vert^2.
\end{equation*}
Since
\begin{equation*}
\mathbb{E}[\bar{l}^{HT}_{\pi}(\widehat{\theta}(t))|\mathcal{F}_t]= \nabla \widehat{L}_N(\widehat{\theta}(t)),
\end{equation*}
we get
\begin{multline*}
\mathbb{E}[ \vert \widehat{\theta}(t+1)-\theta^*\vert ^2 \mid \widehat{\theta}(t)] = \Vert \widehat{\theta}(t)-\theta^*\Vert ^2-2\gamma(t) \nabla F(\widehat{\theta}(t))^T(\widehat{\theta}(t)-\theta^*)\\+\gamma(t)^2\mathbb{E}[\Vert \bar{l}^{HT}_{\pi}(\widehat{\theta}(t))\Vert^2\mid \widehat{\theta}(t)].
\end{multline*}
The strong convexity property gives
\begin{equation*}
\widehat{L}_N(\widehat{\theta}(t))-\widehat{L}_N(\theta^*)\leq \nabla \widehat{L}_N(\widehat{\theta}(t))^{T}(\widehat{\theta}(t)-\theta^*)-\dfrac{l}{2}\Vert \widehat{\theta}(t)-\theta^* \Vert ^2
\end{equation*}
and
\begin{equation*}
\widehat{L}_N(\theta^*) - \widehat{L}_N(\widehat{\theta}(t))\leq -\dfrac{l}{2}\Vert \widehat{\theta}(t)-\theta^* \Vert ^2 ,
\end{equation*}
so that
\begin{equation*}
\l\Vert \widehat{\theta}(t)-\theta^*  \Vert ^2\leqslant \nabla \widehat{L}_N(\widehat{\theta}(t))^{T}(\widehat{\theta}(t)-\theta^*).
\end{equation*}
Combining this inequality with the previous one and taking the expectation, we obtain
\begin{equation*}
\mathbb{E}[ \Vert \widehat{\theta}(t+1)-\theta^*\Vert ^2] \leq \mathbb{E}[\Vert \widehat{\theta}(t)-\theta^*\Vert ^2](1-2\gamma(t) l)+\gamma(t)^2\mathbb{E}[\Vert \bar{l}^{HT}_{\pi}(\widehat{\theta}(t))\Vert^2].
\end{equation*}
Under Assumption 1, we have $\mathbb{E}[\Vert \bar{l}^{HT}_{\pi}(\widehat{\theta}(t))\Vert^2] \leq D$ for some constant $D>0$. Using this bound and iterating the recursion, we finally obtain
\begin{equation*}
\mathbb{E}[\Vert \widehat{\theta}(t+1)-\theta^*\Vert ^2]\leqslant \mathbb{E}[\Vert \widehat{\theta}(1)-\theta^*\Vert ^2] \prod_{j=1}^t(1-2 l \gamma(j))+D\sum \limits_{j=1}^t\gamma(t)^2\prod_{k=j+1}^t(1-2 l \gamma(k)) 
\end{equation*}
with the convention $\prod_{k=t+1}^t(1-2 l \gamma(k))=1$
We now substitute the expression of $\gamma(t)$ and, using the following classical inequalities
\begin{equation*}
1+x\leqslant e^x
\end{equation*}
 and
\begin{equation*}
\log(t+1)-\log(j+1)\leqslant \sum \limits_{k=j+1}^t \dfrac{1}{k},
\end{equation*}
we get
\begin{equation*}
\mathbb{E}\Vert \widehat{\theta}(t+1)-\theta^*\Vert ^2 \leqslant \dfrac{(\mathbb{E}\Vert \widehat{\theta}(1)-\theta^*\Vert ^2+\tilde{D} \sum \limits_{j=1}^t \dfrac{1}{j^{2-2 l \gamma(0)}})}{{(t+1)}^{2 l \gamma(0)}},
\end{equation*}
where $\tilde{D}$ is a positive constant. Since $\gamma(0)>1/(2l)$, we have
\begin{equation*}
\sum \limits_{j=1}^t \dfrac{1}{j^{2-2l \gamma(0)}}\leqslant \dfrac{t^{2 l \gamma(0) -1}}{2 l \gamma(0)-1}
\end{equation*}
and we finally obtain the desired bound.
\end{proof}

{\small

\bibliographystyle{plain}
\bibliography{Arxiv_HTSGD}

}

\end{document}